\ifdefined\pdfoutput
  \pdfoutput=1
\fi
\documentclass[11pt,fleqn]{article}
\usepackage[]{acl} 
\usepackage{placeins}
\usepackage{float}
\usepackage{macros}
\usepackage{cleveref}
\usepackage{booktabs}
\crefname{section}{Sect.}{Sects.}
\Crefname{section}{Sect.}{Sects.}
\title{Efficient Decoding Methods for Language Models on Encrypted Data}

\author{Matan Avitan\textsuperscript{\normalfont1,2}\orcid{0009-0001-0762-9616}\quad 
Moran Baruch\textsuperscript{\normalfont1}\orcid{0000-0003-0615-6164}\quad 
Nir Drucker\textsuperscript{\normalfont1}\orcid{0000-0002-7273-4797}\quad 
Itamar Zimerman\textsuperscript{\normalfont1,3}\orcid{0000-0001-8321-0609}\quad
Yoav Goldberg\textsuperscript{\normalfont2,4}\orcid{0000-0002-6497-829X}\\
\\
\textsuperscript{1}IBM Research
\textsuperscript{2}Bar-Ilan University \\
\textsuperscript{3}Tel Aviv University
\textsuperscript{4}Allen Institute for Artificial Intelligence \\
}

\begin{document}
\maketitle
\begin{abstract}
Large language models (LLMs) power modern AI applications, but processing sensitive data on untrusted servers raises privacy concerns. Homomorphic encryption (HE) enables computation on encrypted data for secure inference. However, neural text generation requires decoding methods like \argmax and \sampling, which are non-polynomial and thus computationally expensive under encryption, creating a significant performance bottleneck. We introduce \cutmax, an HE-friendly \argmax algorithm that reduces ciphertext operations compared to prior methods, enabling practical greedy decoding under encryption. We also propose the first HE-compatible nucleus (top-$p$) sampling method, leveraging \cutmax for efficient stochastic decoding with provable privacy guarantees. Both techniques are polynomial, supporting efficient inference in privacy-preserving settings. Moreover, their differentiability facilitates gradient-based sequence-level optimization as a polynomial alternative to straight-through estimators. We further provide strong theoretical guarantees for \cutmax, proving its convergence via exponential amplification of the gap ratio between the maximum and runner-up elements. Evaluations on realistic LLM outputs show latency reductions of $24\times$–$35\times$ over baselines, advancing secure text generation.
\end{abstract}
\begin{figure*}[t]
    \centering
    \includegraphics[width=0.96\textwidth]{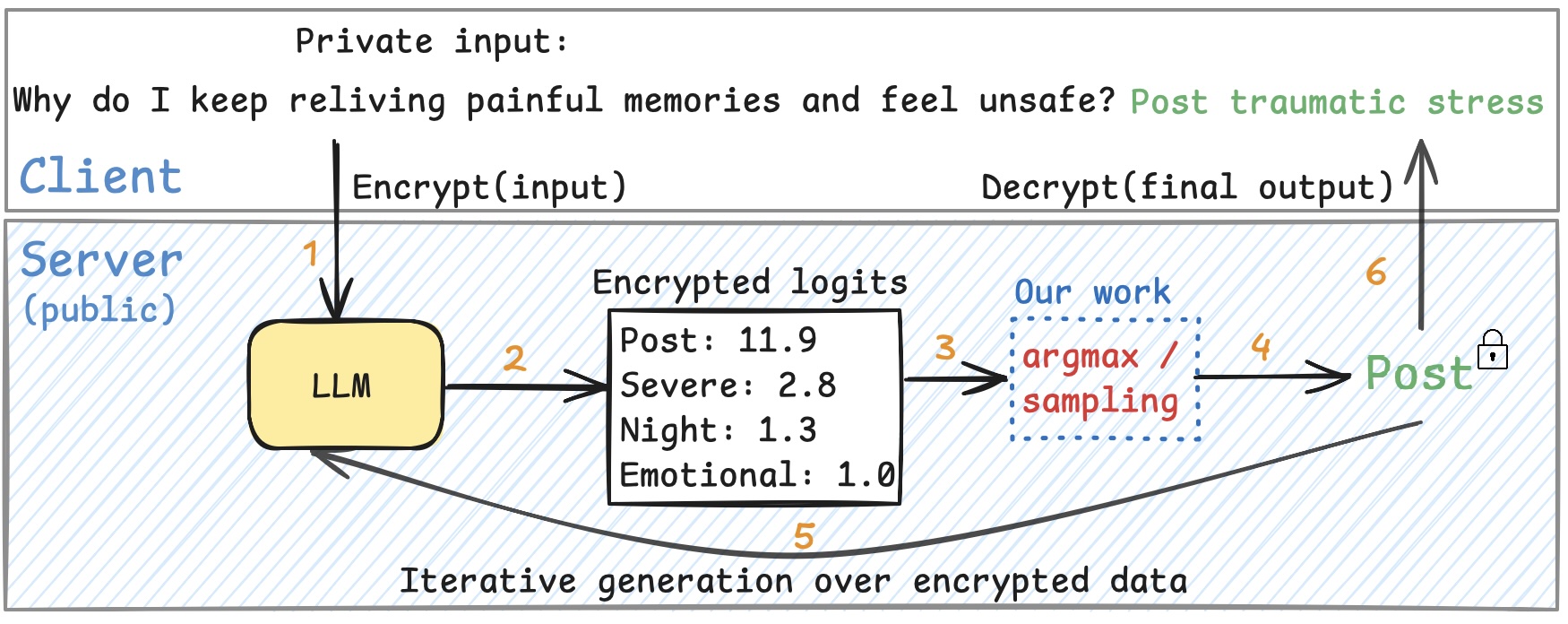}
        \caption{\textbf{Scope}: Secure LLM generation over the clients' encrypted data using HE (model is not necessarily encrypted). Here, standard decoding methods like \argmax and \sampling (\textbf{red}) for selecting the next token are not HE-friendly or considered inefficient. We introduce efficient and scalable HE-friendly methods to evaluate them under encryption.        \label{fig:motivation}}
\end{figure*}
\section{Introduction}
Recent advances in LLMs have enabled the development of powerful AI systems capable of generating fluent text at scale~\citep{brown2020language}. However, deploying these models in real-world applications often involves sending sensitive user data, such as personal messages or medical records, to remote servers, raising significant privacy concerns~\citep{yao2024survey, yan2024protecting}.
Despite major progress in efficient encrypted inference, extending these methods to generative LLMs remains an open challenge due to the non-polynomial nature of decoding operations. Homomorphic encryption (HE) offers a promising solution by allowing computations on encrypted data, ensuring that servers can process queries without accessing their plaintext content~\citep{Gentry2009}. Under HE, users encrypt their inputs, the server performs computations (e.g., running an LLM), and returns an encrypted result that only the user can decrypt.

In this paper, we present differentiable and polynomial \argmax and \nSampling algorithms tailored for encrypted LLM decoding. Unlike prior approaches, such as that of \citet{Bengio2013EstimatingOP}, which use \argmax in the forward pass and estimate gradients in the backward pass via straight-through estimators (STE), our methods are fully differentiable in their plaintext form (without approximations) and remain so under polynomial HE approximations. This differentiability, together with their polynomial form, makes them well suited to AI-privacy settings such as HE; the plaintext formulations are exactly differentiable, and the HE instantiations preserve differentiability via polynomial approximations (see \Cref{subsec:he-cutmax}).
Most HE schemes, including the CKKS scheme used in this work~\citep{ckks2017}, support only polynomial operations like addition and multiplication. This limitation necessitates the design of \textit{HE-friendly} algorithms that rely exclusively on such operations, posing a challenge for text generation tasks that involve complex decoding steps. While secure inference on LLMs using HE has been demonstrated~\citep{zimerman2024power, Castro_2024, transformerArgmax}, extending these techniques to multi-token text generation remains challenging. Decoding methods like \argmax (for greedy decoding) and \sampling (e.g., \nSampling ~\citep{holtzman2019curious}) are computationally inefficient or impractical in the encrypted domain due to their reliance on non-polynomial operations.

State-of-the-art homomorphic \argmax implementations follow two comparison-heavy patterns: a \emph{tournament} tree and an all-to-all \emph{league} schedule, both realized via deep polynomial approximations of \sign \citep{HeCompArgmaxTournament, LRArgmaxTournament, Tyche, pheonix2022}.
A detailed illustration of these designs is provided in \Cref{fig:he-argmax} (\Cref{sec:prev-methods}), which highlights their sequential comparison stages and reliance on costly \sign evaluations.

\cutmax\ removes comparisons altogether. Each of its $T$ iterations applies two global reductions (mean and variance), an inverse-square-root (\invsqr), an elementwise odd power, and a final normalization (\Cref{sec:method}, \Cref{alg:cutmax}, \Cref{subsec:he-cutmax}, \Cref{alg:cutmaxhe}). Consequently, the number of \emph{sequential} stages is a small constant $T$—\emph{weakly dependent} on the vector length $n$ (e.g., the vocabulary size)—whereas tournament and league require $\log_2 n$ and $n$ sequential rounds, respectively, each performing encrypted comparisons via polynomial approximations of $\sign$. Because \invsqr and division (\inv) are shallower and faster than $\sign$ at similar accuracy (\Cref{tab:golds-sign}), \cutmax\ achieves much lower depth and fewer bootstraps per token. Empirically, $T$ is a small constant (e.g., $T{\le}3$–$4$), consistent with the latency gains we observe for vocabularies of $|\mathcal V|\!\approx\!3.3\times 10^4$ and $1.5\times 10^5$ (\Cref{sec:emp_res}, \Cref{tab:empirical_argmax}). Implementation-level costs such as rotations under SIMD packing are deferred to \Cref{subsec:he-cutmax} and \Cref{app:argmaxsimd}.

Our contributions address key bottlenecks in secure text generation
, as illustrated in \Cref{fig:motivation}, and include:
(i) \textbf{CutMax}. An efficient, \gls{HE}-friendly \argmax algorithm that enables practical greedy decoding under encryption with fewer ciphertext operations than previous methods. 
(ii) \textbf{Encrypted nucleus sampling}. The first \gls{HE}-compatible \nucleus (top-$p$) \sampling method, enables stochastic decoding with provable privacy guarantees, via \cutmax. 
(iii) \textbf{Theoretical convergence guarantees}. A formal proof establishing that \cutmax converges via exponential amplification of the gap ratio between the maximum and runner-up elements, providing a rigorous foundation for its rapid convergence in a small number of iterations, weakly dependent on the array size.
(iv) \textbf{Differentiable decoding primitives}. Implementations of \argmax and \nSampling that are real-analytic and fully differentiable in plaintext, leveraging smooth operations like \(1/\sqrt{x}\) (and polynomial-only in their HE approximations) to enable exact gradient-based sequence-level training as a theoretically grounded alternative to straight-through estimators.
By enabling efficient and accurate multi-token generation under full encryption, our work advances the deployment of privacy-preserving LLMs in real-world settings, bridging a critical gap in secure AI systems.
\section{Background: Homomorphic Argmax}\label{sec:argmax}
State-of-the-art homomorphic $\argmax$ implementations for LLM decoding rely on comparison-heavy designs, primarily the \emph{tournament} tree and \emph{league} schedule, both implemented via deep polynomial approximations of the $\textsc{sign}$ function~\citep{HeCompArgmaxTournament, LRArgmaxTournament, Tyche, pheonix2022, Zhang_2024, zhang2025secpesecurepromptensembling}. In the tournament design, approximately $n{-}1$ pairwise comparisons are organized into $\log_2 n$ sequential stages, halving the number of candidates per stage until the maximum remains. The league design performs $\binom{n}{2}=\Theta(n^2)$ comparisons across $n$ sequential rounds, accumulating scores such that the maximum achieves a score of $n{-}1$. Both methods are illustrated in \Cref{fig:he-argmax} (\Cref{sec:prev-methods}). Despite leveraging CKKS SIMD parallelism to execute multiple comparisons within a round, these approaches incur high multiplicative depth due to repeated $\textsc{sign}$ evaluations, leading to costly bootstrapping and limiting scalability for large vocabularies ($n \sim 10^5$). Some works, such as \citet{HeCompArgmaxTournament, Tyche, grivet2021speed}, use TFHE~\citep{tfhe}, restricting input sizes (e.g., $n \le 256$) due to efficiency constraints. Hybrid approaches combine tournament and league methods to exploit SIMD packing~\citep{BGVCompArgmaxCombined, tournamentandleague}, but still rely on slow $\textsc{sign}$ approximations. As shown in \Cref{tab:golds-sign} (\Cref{sec:approximations}), $\textsc{sign}$ evaluations are significantly deeper and slower than \inverse or \inverseSqrt operations, which our \cutmax algorithm uses instead (\Cref{subsec:he-cutmax}). These limitations make prior methods impractical for efficient, privacy-preserving LLM decoding, motivating our polynomial-based approach that eliminates comparisons entirely.

\subsection{HE Preliminaries}
\gls{HE} enables computation on encrypted data without requiring decryption \cite{Gentry2009}. Informally, it is defined as follows: Let $\R_{1}(+, \cdot), \R_2(\oplus, \odot)$ be two rings for the plaintext and ciphertext spaces, respectively. Given plaintext inputs $m_1, m_2 \in \R_1$, their encryption $\heenc{m_i} := \text{Enc}(m_i) \in \R_2$, satisfies the following correctness and homomorphic properties:
\begin{itemize}
    \setlength\itemsep{0.1em}
    \item $ \text{Dec}(\heenc{m_1}) = m_1 + \epsilon$
    \item $ \text{Dec}(\heenc{m_1} \oplus \heenc{m_2}) = m_1 + m_2 + \epsilon$
    \item $ \text{Dec}(\heenc{m_1} \odot \heenc{m_2}) = m_1 \cdot m_2 + \epsilon $
\end{itemize}
where \( \epsilon \) is a small noise introduced by the \gls{HE} scheme, similar to the noise accumulated in floating-point computations. \gls{HE} is \textit{semantically secure} and thus every call to $Enc(m_1)$ results with a new pseudo random ciphertext.
Details about the standard \textit{threat-model} of \gls{HE}-based applications are provided in \Cref{app:threat}.
As \gls{HE} only supports \textit{polynomial operations} (additions and multiplications), standard transformer components like \softmax\ and \layerNorm\ are not directly supported and must be approximated or replaced. Prior-art such as \cite{Dowlin_2016, helayers, baruch2022methodology, baruch2024polynomial, zimerman2024converting, zimerman2024power, Castro_2024, transformerArgmax} have proposed polynomial-friendly adaptations for encrypted inference.
We focus on the \textit{generation process} where new challenges arise: decoding operations such as \argmax and \sampling are not polynomial. 

\section{Method}\label{sec:method}
This section presents our polynomial algorithms for \argmax and \nSampling, tailored for \gls{HE}.
\subsection{The CutMax Algorithm}
Intuitively, \cutmax repeatedly `stretches' the distribution of values and cuts off the lower part, such that after a few iterations only the highest value remains significantly non-zero (see \Cref{fig:cutmax_intuition}). \cutmax is inspired by standardization in statistics (subtracting the mean and dividing by standard deviation), alongside the idea of centering the array around 1; therefore, when taking an odd power of the standardized values, values smaller than the mean would vanish while values greater than the mean would amplify.
The algorithm's fast convergence is due to the right skewness induced by the power operation.
\begin{algorithm}[t!]
    \caption{\cutmax}
    \label{alg:cutmax}
    \begin{algorithmic}[1]
    \setlength{\baselineskip}{1.7em}
    \small
    \Require $X = (X_1,\dots,X_n)$, and $p,T \in \ZZp$, where $p$ is odd and $c \in \RR_+$.
    \Ensure $Z$ - one-hot approx. of $\argmax(X)$.
    \State $Y^{(1)} = X$
    \For{$t=1$ {\bf to} $T$}
        \State $\mu = \frac{1}{n}\sum_{i=1}^n Y^{(t)}_i$
        \Comment{Mean of $Y^{(t)}$}
        \State $\sigma^2 = \frac{1}{n}\sum_{i=1}^n (Y^{(t)}_i - \mu)^2$
        \Comment{Variance of $Y^{(t)}$}
        \State \label{line:cut_y_t} $Y^{(t+1)} = \left(\frac{Y^{(t)}_i - \mu}{c \cdot \sqrt{\sigma^2}} + 1\right)_{i \in [n]}$
        \Comment{Standardization}
        \State $Y^{(t+1)} = \left(Y^{(t+1)}\right)^p_{i \in [n]}$
    \EndFor
    \State $Z = \left(\frac{Y^{(T)}_i}{\sum_j{Y^{(T)}_j}}\right)_{i \in [n]}$
    \State \textbf{return} $Z$
    \end{algorithmic}
\end{algorithm}
   
\Cref{alg:cutmax} provides the pseudo-code for \cutmax, which takes as input a vector $X$ of $n$ elements, an odd integer $p$, a constant scaling factor $c>0$, and the maximum number of iterations $T$. For instance, one instantiation uses $p=11$, $c=5$, and $T=3$. We denote the index set as $[n] = \{1, \ldots, n\}$. \cutmax is an iterative algorithm that runs for at most $T$ iterations or until convergence. Each iteration $t$ consists of two steps: \textbf{standardization} and \noindent\textbf{distance amplification}.
\begin{figure*}[t!]
  \centering
    \includegraphics[width=\linewidth]{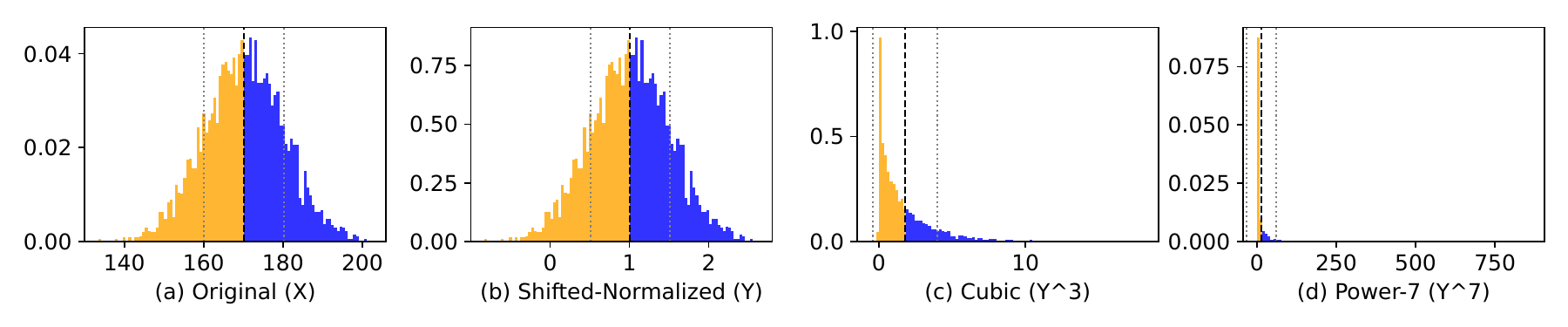}
    \caption{
    \textbf{Illustration of CutMax First Iteration on a Normal Distribution}
      \cutmax algorithm first iteration illustration, applied to the original random variable $X \sim{} \mathrm{Normal}(170, 10)$ (a). We first standardize and shift $X$ using $Y=\frac{X-\mu_X}{c\sigma_X} + 1$ with $c=5$ (b), then raise $Y$ to the power $p=3$ (c) or $p=7$ (d) to induce right skew. The resulting distribution $Z=Y^7$ (d) has a maximum normalized height
    $\max \bigl(Z / \sum_i Z_i \bigr) = 0.0033$, indicating a strong right‐tail skew. A vertical black line indicates the mean of every distribution and gray lines indicate the first standard deviation.}
  \label{fig:cutmax_intuition}
  \vspace{-7pt}
\end{figure*}
\paragraph{Standardization (Lines 3-5).} The normalization process starts by computing the mean $\mu$ of the current vector state $Y^{(t)}$, the variance $\sigma^2$, then the inverse standard deviation $\frac{1}{\sqrt{\sigma^2}}$. Subsequently, it scales $Y^{(t)}$ by subtracting the mean and multiplying by $\frac{1}{c\cdot\sigma}$. The choice of $c$ controls how much we ``flatten'' the distribution before raising the power. A larger $c$ means we divide by a larger number, making the normalized values smaller in magnitude.
\paragraph{Distance Amplification (Line 6).} We raise each normalized element to an odd power $p$, which preserves the sign of $Y_i - \mu$.
Elements above (resp. below) the mean ($Y_i - \mu > 0$) remain positive (resp. negative).
The effect of this power raising is twofold: it shrinks the magnitude of small values and amplifies the larger values, where
\begin{align}
 \left|Y^{(t)}_i\right| < 1 \Longrightarrow \left|Y^{(t)}_i\right| \gg \left|Y^{(t)}_i\right|^p \\
 \left|Y^{(t)}_i\right| > 1 \Longrightarrow \left|Y^{(t)}_i\right| \ll \left|Y^{(t)}_i\right|^p
\end{align}
By appropriately scaling $Y^{(t)}$ before this step, we ensure that all but the largest elements satisfy $|Y_i| < 1$, causing them to shrink, while the largest (and possibly a few near-largest) elements may have $|Y_i| > 1$ and thus grow. This drastically amplifies the gap between the maximum 
and the rest.

\paragraph{Convergence.} As formally proven in Appendix~\ref{app:cutmax-convergence-standalone}, each \cutmax iteration unconditionally grows the gap between the maximum and runner-up elements, yielding exponential amplification of their ratio and ensuring convergence under a mild $\delta$-gap condition on the input.
Beyond this theoretical contraction, \cutmax also prunes much more aggressively
than comparison-based schemes. Unlike the tournament \argmax, which discards only
half of the $\yT$ values per stage, \cutmax zeroes out a \emph{large majority}
each round: empirically, after iteration $t$, the empirical CDF at the mean
satisfies $\mathrm{CDF}_{Y^{(t)}}(\mu^{(t)}) \gg \tfrac{1}{2}$, so far more
than half of the $Y^{(t)}_i$ fall below the mean and are driven toward zero by
the subsequent odd-power step (\Cref{fig:cutmax_convergence_quantile}). 

As established by \Cref{thm:cutmax-convergence}, these properties guarantee
global convergence to the fixed point under mild conditions; for the practical
hyperparameters we use, the limit is nearly one-hot. This explains why only a
small, nearly constant number of iterations (e.g., $T\!\le\!3\text{–}4$) suffices
in practice, regardless of the initial logit distribution (\Cref{sec:emp_res}).

\begin{figure}[t!]
  \centering
  \includegraphics[width=0.91\columnwidth]{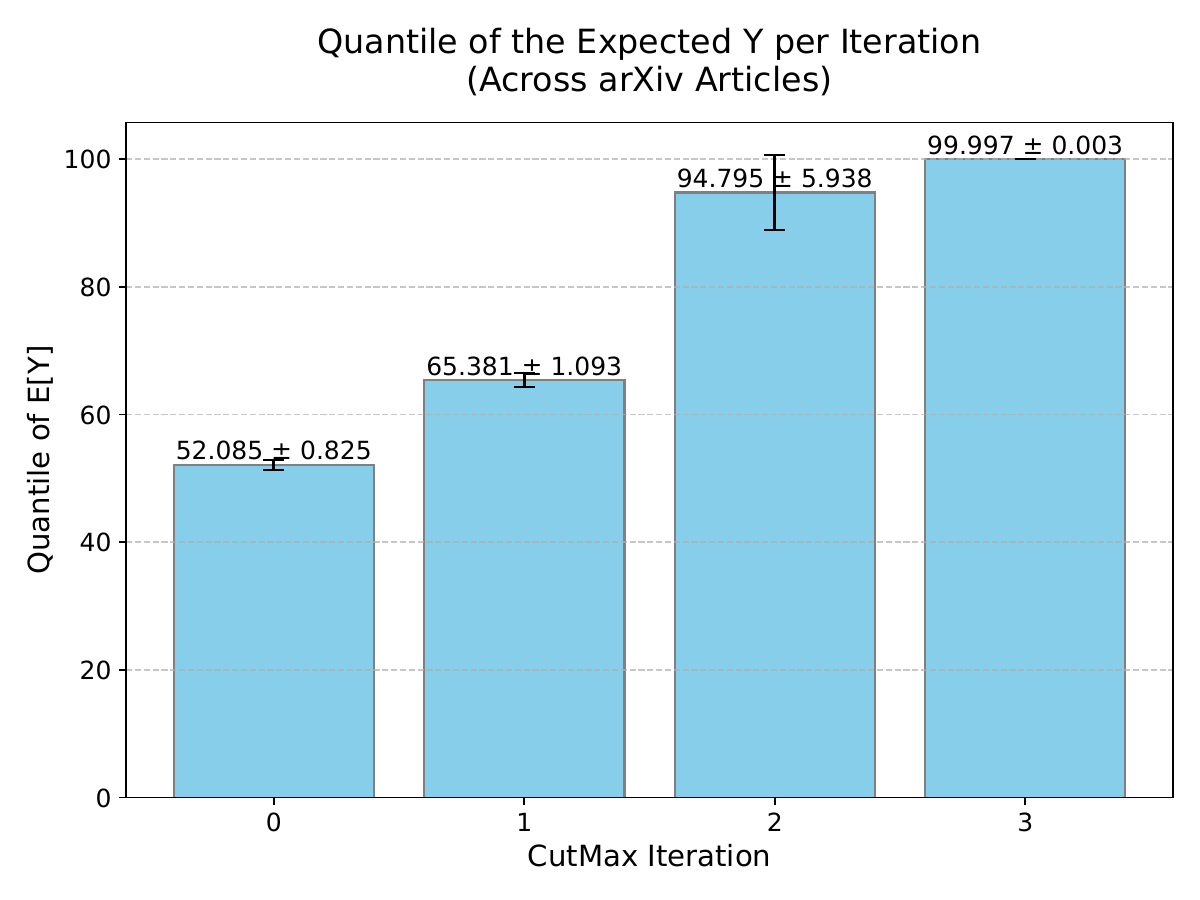}
  \caption{\textbf{Evolution of Logits Below the Mean Across CutMax Iterations} (After \Cref{alg:cutmax}, \Cref{line:cut_y_t}). We ran CutMax for $T=3$ with the best hyperparameters $(p,c)=(19,27)$, found via grid search, on 1,000 arXiv articles passed through GPT-2 (vocabulary size $|\mathcal V|=50{,}257$ \citep{Radford2019LanguageMA}). Bars show the mean$\pm$std of the fraction of entries $\yT_i < \mathbb E[\yT]$ at each iteration $t$.
  }
  \label{fig:cutmax_convergence_quantile}
  \vspace{-7pt}
\end{figure}

\subsection{HE-Friendly CutMax}\label{subsec:he-cutmax}
Building on the HE preliminaries in \Cref{sec:argmax}, we adapt \cutmax to be fully polynomial by approximating $1/\sigma$ (Line 5, \Cref{alg:cutmax}) and $1/\sum_j Y^{(T)}_j$ (Line 7) using Goldschmidt methods \cite{goldschmidt} (details in \Cref{sec:goldschmidt}). This yields \Cref{alg:cutmaxhe}, the FHE variant.
Notably, having a polynomial algorithm is not enough, we also need to ensure that no overflows or precision issues compromise the correctness of the results. Specifically, we must prove the following:
\smallskip
\noindent\textbf{(i) No overflows.} Let $B_{\text{CKKS}}$ be the bound configured by the scheme on the unencrypted inputs, beyond which an overflow may occur. Then, all intermediate values produced by the algorithm must remain below $B_{\text{CKKS}}$.
\smallskip
\noindent\textbf{(ii) No approximation errors.} The \invsqr and \inv approximations assume inputs lie within a certain range, e.g., $[\frac{1}{2^{n}}, 1]$ for some small $n$. If an input $x$ falls in a wider range $[a, b]$, we instead compute $\frac{\invsqr(\frac{x}{b})}{b}$. We must show that $\frac{a}{b} > \frac{1}{2^{n}}$; otherwise, the approximation may yield incorrect results.
\Cref{sec:argmax} reviewed the tournament method that requires $O(n)$ comparisons in $\log n$ sequential stages and the league method that requires $n^2$ comparisons in $n$ sequential stages. In contrast, \cutmax requires only $T$ \invsqr and 1 \inv operations. 
\Cref{tab:golds-sign} in \Cref{sec:approximations} (taken from \cite{asor}) shows that the \invsqr implementation is $1.5\text{--}2.1 \times$ faster than the \sign implementation used in all prior art, at the same error level.
In addition, $T$ is weakly-dependent on $n$ and much smaller. Consequently, \cutmax outperforms prior-art.
Unfortunately, the situation is not that simple. Modern \gls{HE} schemes support \gls{SIMD} operations, meaning a ciphertext can encrypt a vector of $s$ elements instead of just a single element. In this case, multiplications ($\odot$), additions ($\oplus$), and subtractions ($\ominus$) are applied element-wise homomorphically over the vector. Furthermore, a homomorphic rotation by $\ell < s$ positions is also available \citep{ckks2017}.
\begin{algorithm}[t!]
    \caption{\cutmaxhe~- Encrypted \cutmax}
    \label{alg:cutmaxhe}
    \begin{algorithmic}[1]
    \small
    \Require $\heenc{X} = \heenc{(\XX_1,\dots,X_n)}$, and $p,T \in \ZZp$, where $p$ is odd, and $c \in \RR_+$.
    \Ensure $\heenc{Z}$ one-hot approx. of $\argmax(X)$.
    \State $\heenc{Y^{(1)}} = \heenc{X}$
    \For{$t=1$ {\bf to} $T$}
        \State $\heenc{\mu} = \frac{1}{n}\rotateandsum(\heenc{Y^{(t)}})$
       
        \State $\heenc{\sigma^2} = \frac{1}{n}\rotateandsum\left((\heenc{Y^{(t)}} \ominus \heenc{\mu})^2\right)$
        \setlength{\baselineskip}{2.4em}
       
        \State $\heenc{\sigma^{-1}} = \invsqr(\heenc{\sigma^2})$
       
        \State $\heenc{Y^{(t+1)}} = \left(\frac{1}{c} \heenc{\sigma^{-1}} \odot (\heenc{Y^{(t)}} \ominus \heenc{\mu})\right) \oplus {1}$
        \State $\heenc{Y^{(t+1)}} = \left( (\heenc{Y^{(t+1)}}\right)^p$
    \EndFor
    \State $\heenc{Z} = \heenc{Y^{(T)}} \odot \inv\left(\rotateandsum(\heenc{Y^{(T)}})\right)$
    \State \textbf{return} $\heenc{Z}$
    \end{algorithmic}
\end{algorithm}
We defer the discussion of how \gls{SIMD} impacts latency to \Cref{app:argmaxsimd}, and here we complete the picture with an implementation of \cutmax over \gls{HE}. \Cref{alg:cutmaxhe} is the \gls{HE} approximation variant of \cutmax, where the input vector $X$ is encrypted. For brevity, we assume $n=a \cdot s$, $a \in \ZZ_+$, and write the algorithm so that modern compilers such as HELayers \cite{helayers} can automatically implement it.
Lines 3, 4, and 8 of \Cref{alg:cutmax} are efficiently computed in \Cref{alg:cutmaxhe} using the \rotateandsum method (see \cite{he4ds}), which takes an encrypted vector $v = (v_1, \ldots, v_s)$ and returns a vector of $s$ elements, each containing the sum $\sum v_i$. This requires $\log s$ rotations and additions.
Multiplying by a plaintext scalar ($\frac{1}{n}$ at lines 3,4 and $\frac{1}{c}$ at line 6) is considered a cheap \gls{HE} operation as well as the homomorphic subtractions at lines 4 and 6. Raising to the power of 2 (Line 4) and $p$ (Line 6) require 1 and $\log p$ multiplications, respectively. Finally, the heaviest operations used are the \invsqr (Line 5) and \inv (Line 8) that use the Goldschmidt approximations (\Cref{sec:goldschmidt}).
Denote by $M_{\inv}$ (resp. $D_{\inv}$) and $M_{\invsqr}$ (resp. $D_{\invsqr}$) the number of multiplications (resp. multiplication depth) required by the \inv and \invsqr approximations over one ciphertext, respectively. Then \cutmaxhe executes
\begin{equation}
T \cdot (M_{\invsqr} + 5 + \log p) + M_{\inv} + 1
\end{equation}
multiplications with a total depth of
\begin{equation}
T \cdot (D_{\invsqr} + 5 + \log p) + D_{\inv} + 1 \,
\end{equation}
and $O(T \log s)$ rotations.
\subsection{HE-Friendly Nucleus Sampling} \label{method:he_nucelus}
In text generation, language models produce a probability distribution over a large vocabulary at each time step. The next token is selected from this distribution using a decoding strategy. The simplest strategy is greedy decoding, where the token with the highest probability (\argmax) is chosen. More commonly, however, high-quality text generation relies on stochastic decoding methods that introduce controlled randomness to improve fluency, diversity, and coherence.
Two widely used stochastic decoding methods are \topKSampling \cite{Fan_2018} and $\operatorname{nucleus}$ $
(\operatorname{top-p})$ \sampling \cite{Holtzman_2020}. In \topKSampling, only the $k$ most probable tokens are considered, and one is sampled according to its probability. In \nSampling, the model selects the smallest set of top-ranked tokens whose cumulative probability exceeds a threshold $p$, and samples from this adaptive set. These techniques prevent degeneration and repetition, making them essential in modern LLM-based systems.
\vspace{-3pt}
\paragraph{Sampling under HE.}
As established, our \cutmax algorithm is polynomial and thus compatible with HE; we confirm its correctness empirically in \Cref{sec:emp_res}. Whereas \cutmax handles deterministic decoding, we next extend it to stochastic sampling.
The first naive attempt is to use the \textit{inverse transform sampling} method, which generates random samples from a given probability distribution using its \gls{CDF}: Given an encrypted \gls{PDF} $\XX$ with \gls{CDF} $F_\XX$, the method samples a plaintext value $\UU \sim \mathrm{Uniform}(0,1)$, and returns $x = F_\XX^{-1}(u)$, where $F_\XX^{-1}$ is the inverse \gls{CDF} ($\mathrm{iCDF}$), and $x$ is a sample from $\XX$. This sampling method is inefficient as it requires at least one \gls{HE} heavy comparison (\Cref{sec:approximations}) to sample an element.
\vspace{-3pt}
\paragraph{Using Gumbel Distribution.} We can improve upon the above by using the $\mathrm{Gumbel}$ distribution. We begin by recalling the relationship between the $\mathrm{Gumbel}$ and $\mathrm{Uniform}$ distributions, as well as the \emph{Gumbel-max trick}.
\begin{definition}\cite{gumbel1954statistical}
Let $\UU \sim{} \mathrm{Uniform}(0, 1)$, then $\GG = -\log(-\log \UU)$ has
a Gumbel distribution $\GG \sim{}
\mathrm{Gumbel}(\mu=0, \beta=1)$.
\end{definition}
\begin{lemma}[Gumbel-Max Trick \cite{ maddison2014}]\label{lem:gumbel}
For $X \in R^k$, let $\GG_i \sim \mathrm{Gumbel}(0,1)$ be independent for each $i$. Then the random variable $\YY = \argmax_{i} \left( X_i + \GG_i \right)$
follows a categorical distribution with probabilities given by,
\[\mathbb{P}(Y = i) = \frac{e^{X_i}}{\sum_j e^{X_j}} = \softmax(X_i) \,.
\]
\end{lemma}
Using the above, our sampling method is as follows: for a given encrypted logits array $\XX \in \RR^n$, sample unencrypted $\UU \in \mathrm{Uniform}(0,1)^n$ and convert it to $\GG \sim{} \mathrm{Gumbel}(0,1)^n$, then compute $\cutmax(\XX+\GG)$ over \gls{HE}, which only requires fast additions and one call to \cutmax.
\vspace{-3pt}
\paragraph{Efficient Nucleus Sampling.}
While our Gumbel-based method allows efficient sampling over \gls{HE}, we can do better. We aim to develop an algorithm inspired by \cite{maddison2014} that will enable us to sample only from the top of the distribution, as in \nSampling, so that words from the tail of the logits distribution could not be sampled, and potentially disrupt the coherence of the generated text. To this end, we aim to sample $\GG$ from a \gls{PDF} that satisfy the condition $\PP(\GG > p) = p$. Luckily, the $\mathrm{Beta}(\alpha, \beta)$ distribution has this property when $\beta=1$.
\begin{lemma}\label{lem:beta}
Let $\GG \sim{} \mathrm{Beta}(\alpha,1)$, then $\PP(\GG > p) = p$ for $\alpha = \frac{\ln(1-p)}{\ln(p)}$ \,.
\end{lemma}

\begin{proof}
The \gls{CDF} of $\mathrm{Beta}(\alpha, 1)$ is $F_{\mathrm{Beta}(\alpha, 1)}(x) = x^{\alpha}$.
The requirement $\PP(\GG > p) = p$ holds if and only if $F_{\mathrm{Beta}(\alpha, 1)}(p) = 1 - p$. Thus,
$p^{\alpha} = 1 - p$, or alternatively, $\alpha = \frac{\ln(1-p)}{\ln(p)}$.
\end{proof}
\begin{observation}
\label{obs:nucleus}
\Cref{lem:beta} targets $\PP(\GG > p) = p$. We can generalize this result by asking $\PP(\GG > p) = q$ for any $q \in [0,1]$, which yields $\alpha = \frac{\ln(1-q)}{\ln(p)}$
\end{observation}
\begin{algorithm}[t]
\caption{Nucleus One-Shot Sampling}
\label{alg:nucleus}
\begin{algorithmic}[1]
    \small
    \setlength{\baselineskip}{1.5em}
    \Require Logits $\heenc{X}$, $\XX \in \RR^n$; Nucleus mass $p\in(0,1)$.
   
    \Ensure Encrypted one‐hot sample $\heenc{z}$.
   
    \State $\alpha = \frac{\ln(1-p)}{\ln(p)}$
   
    \State $\beta = 1$
    \State $\UU \gets \mathrm{Uniform}(0,1)^n$
   
    \State $\GG = \left(F^{-1}_{\mathrm{Beta}(\alpha,\beta)}(\UU_i)\right)_{i \le 1,\ldots,n}$
    \Comment{Unencrypted noise}
   
    \State $\heenc{\tilde \XX} = \heenc{\XX} \oplus \GG$
    \Comment{Perturb logits}
   
    \State \Return $\cutmax(\heenc{\tilde \XX})$
    \end{algorithmic}
\end{algorithm}
\Cref{alg:nucleus} introduces a single‐shot, nucleus-restricted sampling routine based on \Cref{lem:beta}. It starts by draw $\GG\sim{}\mathrm{Beta}(\alpha, \beta)$ (Line 4). Subsequently, we form the ``cut-noise'' logits at Line 5, so that only those tokens whose cumulative mass exceeds $p$ can potentially become the maximum. Finally, the algorithm executes a single invocation of \cutmax at Line 6 to
homomorphically recover the one-hot encrypted sample from the nucleus set.
To summarize, we introduced a one-shot nucleus sampling algorithm whose polynomial formulation requires only a single \cutmax evaluation under HE, while maintaining exact distributional behavior and full HE compatibility.

\subsection{Differentiability for Sequence Optimization}
A key property of \cutmax and our \nSampling is their composition from polynomial operations (additions, multiplications, means/variances, \invsqr, and elementwise powers). In the plaintext domain (Algorithm~\ref{alg:cutmax}), these operations are naturally differentiable without any approximations: for example, the \inverseSqrt \(1/\sqrt{x}\) is \(C^\infty\) smooth and differentiable for \(x > 0\), enabling exact gradient computation via automatic differentiation. This contrasts with prior tournament/league methods, which inherently rely on the discontinuous SIGN function—even in plaintext, SIGN requires smoothing or approximation for differentiability, often leading to high-degree polynomials that can introduce gradient instabilities.

In HE, both approaches use polynomial approximations, but CutMax approximates smooth functions like \(1/\sqrt{x}\), which can be achieved with lower-degree polynomials compared to approximating the sharp discontinuity of SIGN (as evidenced by the lower multiplicative depths in Table~\ref{tab:golds-sign}). Theoretically, polynomial composition preserves differentiability, providing a rigorous foundation for exact end-to-end gradients through the full forward pass.

This stands in contrast to non-differentiable ops like discrete \argmax or \sampling, which require approximations like straight-through estimators (STE; \citet{Bengio2013EstimatingOP}) during backpropagation—where gradients are passed `straight through' the hard \argmax as if it were a soft identity. STE introduces bias and can lead to optimization instabilities in sequence tasks (e.g., reinforcement learning from human feedback or sequence fine-tuning).

Theoretically, our methods serve as drop-in polynomial surrogates: during training, replace hard decoding with \cutmax (for greedy) or \nSampling based upon \cutmax (for \sampling), and compute exact gradients end-to-end. For instance, in a sequence loss \(\mathcal{L}(\mathbf{y}, \hat{\mathbf{y}})\) where \(\hat{\mathbf{y}}\) is generated via decoding, the gradient \(\partial \mathcal{L} / \partial \mathbf{x}\) (w.r.t.\ input logits \(\mathbf{x}\)) flows precisely through the polynomial chain. While we leave empirical validation (e.g., on BLEU/ROUGE improvements in fine-tuning) to future work, the theoretical differentiability—rooted in smooth plaintext operations and polynomial preservation—positions \cutmax as a promising alternative for stable, gradient-based sequence optimization.

\section{Empirical Evaluation}\label{sec:emp_res}
To evaluate our methods, we use the \texttt{ccdv/arxiv-summarization}~\cite{cohan-etal-2018-discourse} dataset of paper abstracts, sampling $N=100$ articles unless otherwise specified. We report results on realistic LLM logits from models including QWEN2.5-0.5B ($|\mathcal{V}|\approx150\mathrm{K}$), Mistral-7B ($|\mathcal{V}|\approx33\mathrm{K}$), and GPT-2 ($|\mathcal{V}|\approx50\mathrm{K}$).
\subsection{Plain-Text CutMax}
We evaluate our \cutmax\ in the unencrypted domain on the QWEN2.5-0.5B model (vocabulary size $|\mathcal{V}|\approx150\mathrm{K}$), using the \texttt{ccdv/arxiv-summarization}~\cite{cohan-etal-2018-discourse} dataset of paper abstracts. We sampled $N=100$ articles and, for each, retrieved the next‐token logits via one‐step greedy decoding (\textit{greedy argmax} on the plaintext logits), yielding 100 ground‐truth tokens. For each prompt we then ran \cutmax\ with $T\in\{2,3,4,5\}$ iterations, grid‐searching over odd powers $p\in\{7,9,\dots,19\}$ and scaling factors $c\in\{3,5,\dots,39\}$ to find the smallest \emph{$(p,c)$} that still achieves perfect next‐token recovery.
\paragraph{Convergence vs.\ \#iterations.}
\Cref{tab:empirical_T} summarizes, for each $T$, the best $(p,c)$ choice, the resulting average maximum normalized score after $T$ iterations,
$
\argmaxi = Z_i / \sum_{j=1}^\mathcal{V} Z_j
$,
and the recovery rate (fraction of times $\argmaxiT$ matches the true token). The results show rapid convergence: even with just $T=2$ iterations, \cutmax\ concentrates $\approx95\%$ of the mass on the correct token. By $T=4$, the recovery exceeds $99.99\%$, and by $T=5$, it achieves exact recovery in all cases.
\begin{table}[t]
  \centering
  \small
  \setlength\tabcolsep{3pt}
  \renewcommand{\arraystretch}{0.98}
  \caption{Convergence of \cutmax\ on QWEN2.5-0.5B ($|\mathcal V|\approx150$K) over 100 arXiv abstracts. We report the mean of the per-prompt $\argmaxiT$, the best (i.e.\ largest) $\argmaxiT$ observed across the 100 prompts with exact next-token recovery of 100\% in all cases.}
  \label{tab:empirical_T}
  \begin{tabular}{@{}cccc@{}}
    \toprule
    Iter. $(T)$
    & Best $(p,c)$
    & $\!\mathbb{E}[\argmax_i]\!\uparrow$
    & $\max\!\bigl(\argmax_i\bigr)\!\uparrow$ \\
    \midrule
     $2$ & $(19,5)$ & $0.952$ & $0.987$ \\
     $3$ & $(19,27)$ & $0.987$ & $0.994$ \\
     $4$ & $(13,15)$ & $0.999996$ & $1.000$ \\
     $5$ & $(9,15)$ & $1.000$ & $1.000$ \\
    \bottomrule
  \end{tabular}
  \vspace{-7pt}
\end{table}
\paragraph{Sparsity and Accuracy.}
Across all 400 \cutmax\ invocations (100 prompts × 4 values of $T$), we achieve 100\% next‐token recovery. Moreover, after convergence (e.g., $T\ge4$), the distribution is effectively one‐hot: out of $\approx150\mathrm{K}$ vocabulary items, only the true token carries non-negligible mass.
These results confirm the theoretical analysis of \Cref{sec:method}: \cutmax\ drives a highly skewed distribution in just a handful of 
iterations, making it an excellent practical choice for homomorphic greedy decoding under CKKS‐style encryption.
\begin{table*}[h!]
  \centering
  \small
  \setlength\tabcolsep{3pt}
  \renewcommand{\arraystretch}{0.9}
  \caption{\textbf{Homomorphic $\argmax$ latency and accuracy.}
We benchmark \cutmax\ on two realistic LLM logit tensors—Mistral‑7B
($\vert\mathcal{V}\vert\!=\!32{,}768$) and Qwen‑2.5‑0.5B
($\vert\mathcal{V}\vert\!=\!151{,}936$)—using 100 prompts from the \textsc{arXiv}‑
summarisation set.
Hyper‑parameters are fixed to T=4 with p=[16,4,4,6] and c=[5,3,3,3] due to chain index optimization and numerical stability. We report the mean of the average of maximum values, the average latency, the accuracy of slot predictions, the arrays’ sizes, and the device used.}
  \label{tab:empirical_argmax}
  \begin{tabular}{@{}l l c c c c c l@{}}
    \toprule
    Algorithm
    & Dataset
    & Model
    & Array
    & $\mathbb{E}[\max_i Z_i]\!\uparrow$
    & Avg.\ Latency $\downarrow$
    & Accuracy $\uparrow$
    & Devices \\
    & & & Length & & (secs) & & \\
    \midrule
    \cutmax & arXiv & Mistral7B-v0.3 & 32,768 & $0.994$ & $3.373$ & $100.0\%$ & $1\times$ NVIDIA H100 \\
    \cutmax & arXiv & Qwen2.5 & 151,936 & $0.989$ & $4.607$ & $100.0\%$ & $1\times$ NVIDIA H100 \\
    \midrule
    \multicolumn{8}{@{}l}{\textbf{Baselines}} \\
    Nexus & GLUE & BERT & 30,522 & \textemdash & $79.36^*$ & \textemdash & $4\times$ Tesla A100 \\
    Nexus & GLUE & Llama-3 & 128,256 & \textemdash & $162.8^*$ & \textemdash & $4\times$ Tesla A100 \\
    \bottomrule\\
  \end{tabular}\\
  $^*$ Nexus \cite{transformerArgmax} Table 2: "We batched 32 inputs in total ... Runtime is the amortized latency of each input.", originally reported $(2.48{,}5.09)$ when multiplied by 32 we get $(79.36, 162.8)$, respectively.
\end{table*}
\subsection{HE CutMax}
We evaluated our FHE \cutmax implementation under the CKKS scheme using HELayers~\citep{helayers}. Experiments were conducted on LLM logits from two models: Qwen-2.5-0.5B and Mistral-7B, with vocabulary sizes of 151,936 and 32,768, respectively. For each model, we used 100 examples and set following hyperparameters:
$c = [5,3,3,3],\quad
p = [16, 4, 4, 6], \quad
T = 4$, 
chosen to optimize the CKKS chain index and ensure numerical stability. We report the mean of the per-example maximum values, average latency, and slot prediction accuracy. As a baseline, we compare our method against the recently proposed HE-friendly \argmax algorithm by~\citet{zhang2024secure}. Results are reported in \Cref{tab:empirical_argmax}, demonstrating that our method is both accurate and efficient. From an accuracy perspective, our method does not alter the model's predictions in any of the tested cases, achieving 100\% accuracy. For efficiency benchmarking, our method significantly reduces the latency of \argmax compared to Nexus~\cite{transformerArgmax}, by several orders of magnitude for comparable vocabulary sizes. Specifically, for a vocabulary size of approximately 30K, our method reduces latency from 79.56 seconds to 3.373 seconds. For a larger vocabulary of approximately 150K, our method reduces latency from 162.8 to 4.607. %
Note that we chose Nexus as our baseline even though there are newer papers on LLM inference such as \cite{zimerman2024power, powerformer, moai, thor, kei2025shaft}, because these did not include an \argmax benchmark or did not utilize \argmax at all.

\subsection{Nucleus Sampling}
\Cref{fig:sampling_example} qualitatively demonstrates our method's effectiveness, showing a sharply truncated distribution that confirms all probability mass remains within the top-$p$ nucleus.
To verify that our one‐shot nucleus sampler never selects tokens outside the intended top-$p$ set, we measured the \emph{violation rate} as the fraction of draws falling outside the nucleus, over $S=100$ prompts, with $N=1000$ samples per prompt and $p=0.9$. We compare standard Gumbel-Max sampling against our Nucleus($\beta$--cut) method. Results are summarized in \Cref{fig:nucleus_violation}: Gumbel-Max violates the top-$p$ constraint in $9.3\%\pm2.4\%$ of draws, whereas Nucleus ($\beta$‐cut) has \emph{zero} violations.
\begin{figure}[!b]
  \centering
  \includegraphics[width=1.0\linewidth]{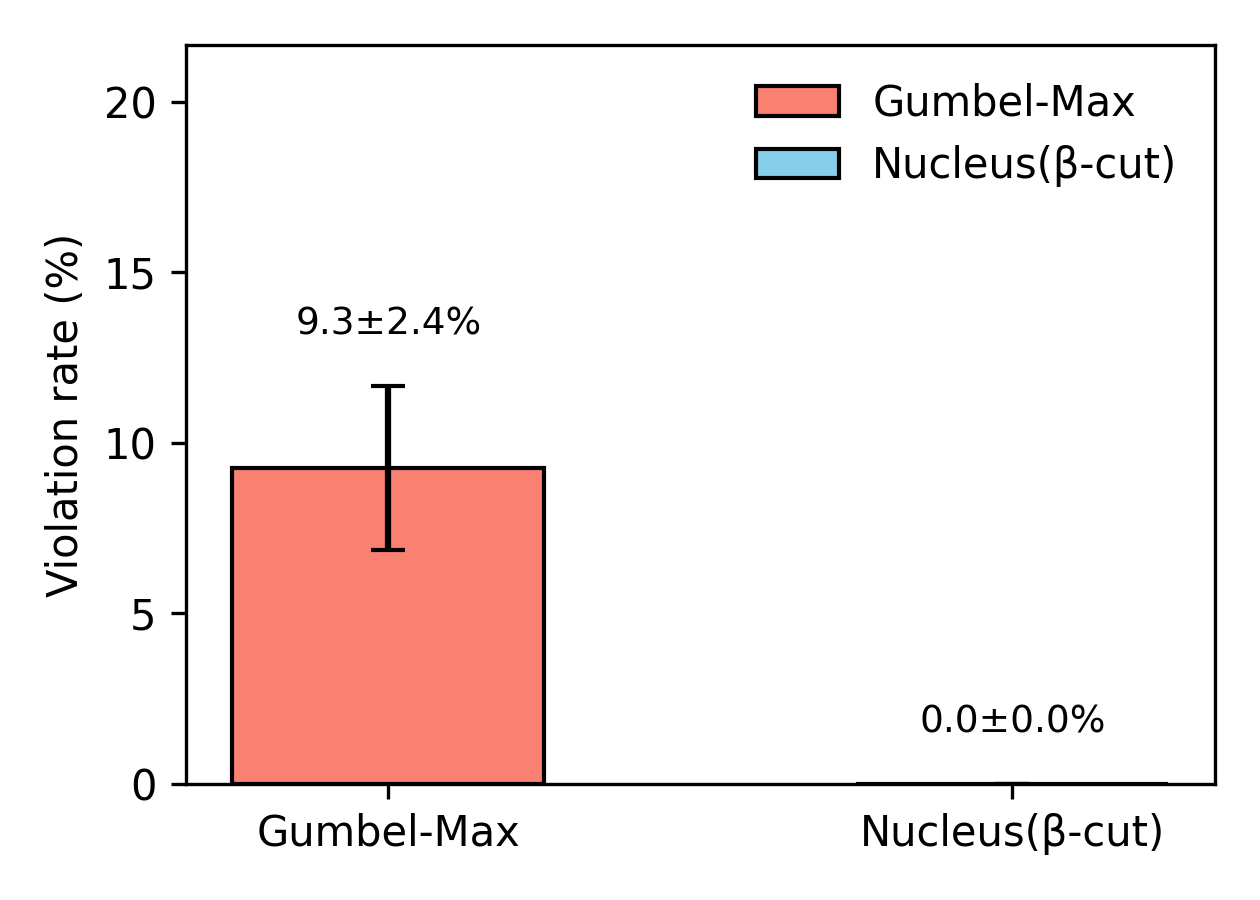}
  \caption{\textbf{Violation Rates in Nucleus Sampling Methods}: Violation rate (\%) of sampling outside the top-$p$ nucleus for Gumbel-Max and Nucleus($\beta$-cut), averaged over 100 prompts with 1000 draws each ($p=0.9$). Error bars denote one standard deviation. }
  \label{fig:nucleus_violation}
\end{figure}
\section{Conclusions}\label{sec:conc}
We solve a critical bottleneck in privacy-preserving AI: adapting LLM decoding for homomorphic encryption, where non-polynomial operations like \argmax and \sampling are computationally prohibitive. We introduce \textbf{CutMax}, a novel, HE-friendly and theoretically grounded \argmax algorithm that replaces slow, comparison-based methods with an efficient iterative polynomial process. Crucially, \cutmax achieves 100\% accuracy and executes in just a few seconds on large vocabularies ($|\mathcal{V}|\!\approx\!150\mathrm{K}$), making secure greedy decoding practical for LLM inference for the first time. Building on this, we present the \textbf{first HE-compatible nucleus (top-p) sampling method}, which leverages \cutmax with the idea of noisy inverse transform sampling to provably restrict sampling to the desired token set. Together, these fully polynomial algorithms provide a complete and efficient framework for both greedy and stochastic decoding over encrypted data, addressing a major barrier to the deployment of secure LLMs in real-world applications.
\section*{Limitations}\label{sec:limit}
Our work focuses on optimizing the decoding stage of encrypted token generation. While this is a critical component, it does not capture the full computational cost of encrypted inference. We view this work as a step toward narrowing the still-substantial gap between encrypted and plaintext generation, and we believe that combining our techniques with future advances, such as attention mechanisms, can help enable practical, privacy-preserving generative AI.
Our implementation uses the CKKS scheme, which is well-suited for approximate arithmetic and widely adopted in encrypted machine learning. While the approach is conceptually compatible with other HE schemes, extending it may require engineering adaptations. As the field evolves, e.g., through new HE schemes or hardware support, our techniques can be adapted accordingly.
Our paper focuses on improving the efficiency of LLM decoding methods under encryption. Therefore, it preserves the behavior of the underlying generative model, including any biases or factual inaccuracies it may contain. While encryption protects user data, it also limits visibility into model behavior, making post-hoc filtering and auditing more challenging. We view this as a crucial direction for future research in secure and responsible AI.
\begin{figure}[!h]
  \centering
  \includegraphics[width=0.92\columnwidth]{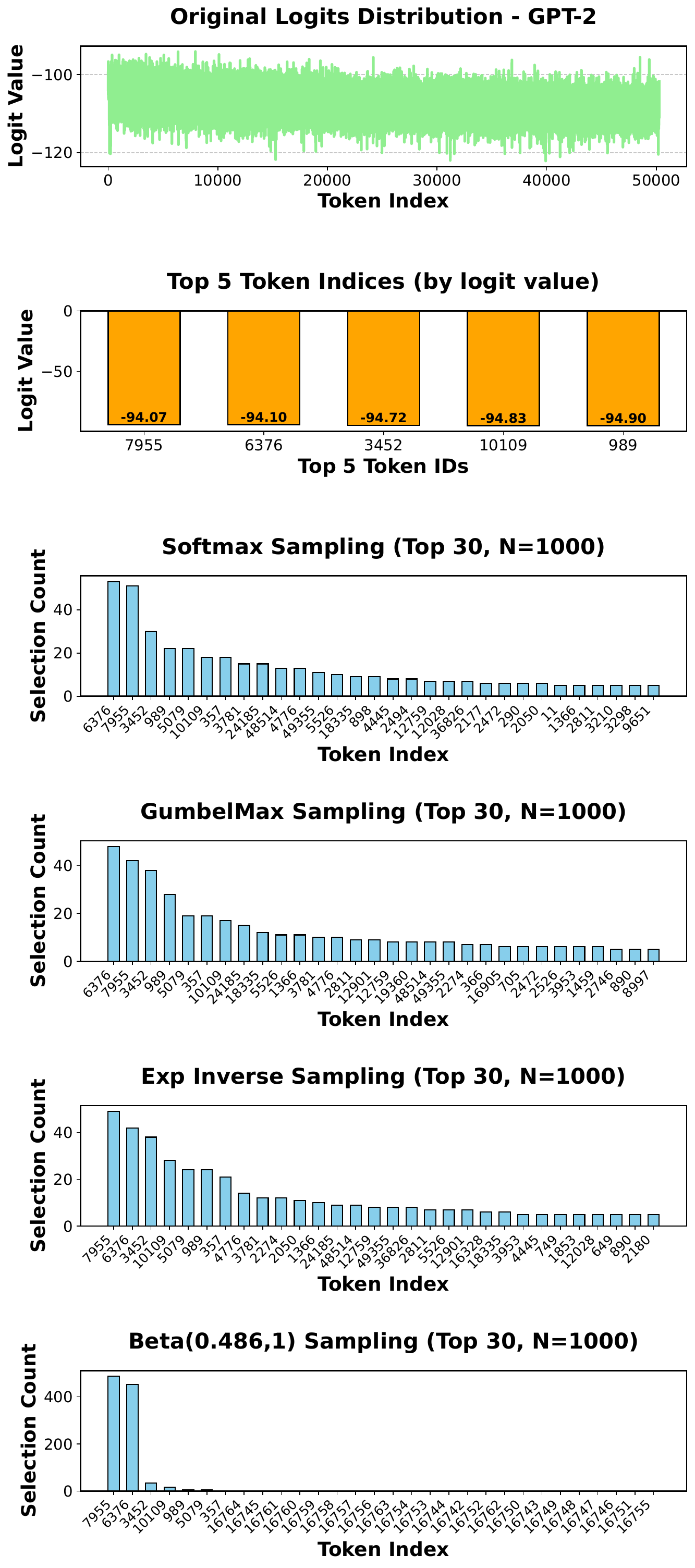}
    \caption{\textbf{Comparison of Sampling Methods on GPT-2 Last-Token Decoding:}
The top strip shows the complete 50,257‑dimensional logit vector (only its outline is visible at this scale), and the second strip zooms in on the five largest logits. The three lower strips plot, for 1,000 repeated draws, histograms of the 30 most frequently chosen token indices produced by four different samplers: standard \softmax\ at temperature 1, the classical Gumbel‑Max trick, inverse‑transform sampling on the normalized exponentials (“Exp‑inverse”), and our one‑shot \emph{Beta‑cut nucleus} sampler, which perturbs each logit with i.i.d.\ $G_i\!\sim\!\mathrm{Beta}(\alpha,1)$ noise using $\alpha=0.486$ computed from \Cref{obs:nucleus} for $(p,q)=(0.9,0.95)$. Whereas the first three methods still allocate non‑negligible mass to tail tokens, the Beta‑cut histogram is sharply truncated, confirming that all probability remains inside the top‑$p$ nucleus exactly as required by the HE‑friendly design of \Cref{method:he_nucelus}.}
  \label{fig:sampling_example}
\end{figure}
\section*{Ethical Considerations}\label{sec:ethical}
Our work aims to enhance privacy in text generation by enabling efficient LLM decoding on encrypted inputs. This can reduce the exposure of sensitive user data when interacting with large language models. However, our method preserves the behavior of the underlying model, including any biases, inaccuracies, or harmful content it may produce. Moreover, since encrypted inference limits visibility into intermediate states, it may complicate post-generation auditing and filtering. We view secure model auditing and bias mitigation as important complementary directions for future research.
\bibliographystyle{acl_natbib}
\bibliography{custom}
\appendix
\onecolumn
\section*{Appendix}
\section{Threat model}\label{app:threat}
There are two commonly used threat-models for secure inference and generative AI for LLMs over HE: 1) using encrypted weights; or 2) using encrypted input samples (queries). When considering a 2-party scenario with \gls{HE}, it is commonly assumed that one party is the (untrusted, semi-honest) server who holds the model (encrypted or not) and the other party is the data-owner who performs the query and would like to avoid revealing the query to the server. The decision of whether the server holds the model encrypted or not depends on whether a third-party the model-owner agrees to share the data with the server. Our proposed solution is orthogonal to the above decision, which eventually only affects latency. All communications are performed using TLS 1.3 and privacy attacks such as model extraction attack and membership inference attacks require an extra mechanism such as differential privacy, which is also orthogonal to this work. For brevity, we have decided to refer the reader to \cite{he4ds}[Chapter 3] that further describes this security model.
\section{Approximations for \gls{HE}}
\label{sec:approximations}
Since HE supports only polynomial operations, functions like $\max$, $\sign$, $\inv$, and $\invsqr$ must be approximated. A common formulation for $\max(a,b)$ is:
\begin{equation}\label{eq:max}
    \max(a,b) = \dfrac{a+b}{2} + \dfrac{\sign(a-b) \cdot (a-b)}{2},
\end{equation}
where $\sign(x) = 1$ when $x > 0$ and $-1$ otherwise. An approximation of \sign is given, e.g., in \cite{asor}. To make the results of \sign suitable as an indicator function (as in \Cref{eq:max}), it is common to translate the \sign range from $\{-1,1\}$ to $\{0,1\}$ using $\frac{1 + \sign(x)}{2}$.
The $\max$ and \sign functions are used by prior-art \argmax implementations. In contrast, \cutmax uses \inv and \invsqr approximations.
\subsection{Goldschmidt}\label{sec:goldschmidt}
The \textit{Goldschmidt} inverse algorithm \cite{goldschmidt} is an iterative process for computing the function $f(x)=\frac{1}{x}$. \cite{goldapprox} showed that this algorithm can be used to approximate inputs in the range $[0,2]$ effectively using the equation
\begin{equation}
f(x) = \prod_{i=0}^{d} \left(1+(1-x)^{2^i}\right)
\end{equation}
Specifically, they proved that for $( x \in \left[\frac{1}{2^n}, 1\right])$, it requires $(d = \Theta(\log \alpha + n) )$ iterations to converge, with an error bound of $(2^\alpha )$. \cite{panda} showed an efficient implementation for the Goldschmidt inverse square root (\invsqr) algorithm over \gls{HE}. Subsequently, \cite{asor} reported improved measurements of the inverse and \invsqr implementations under CKKS. \Cref{tab:golds-sign} repeats some benchmarks for the \sign and \invsqr functions at different approximation levels ($\alpha$). One immediate conclusion is that the \invsqr implementation is $(1.5 - 2.1 \times )$ faster compared to the \sign implementation for the same level of errors. One reason is that \invsqr has a lower \textit{multiplication depth} (defined below) and thus requires fewer bootstraps.
\paragraph{Multiplication Depth.} In CKKS, a ciphertext must be refreshed via a costly \textit{bootstrap} operation after every $L$ sequential multiplications, where $L$ is a scheme parameter. Since bootstrapping is significantly slower than other operations, a key optimization goal is to minimize its use. The longest chain of multiplications in a function—its \textit{multiplication depth}, is therefore a critical target for reduction.
\begin{table}[ht!]
\centering
\caption{A sample benchmarks of \cite{asor} conducted on AMD EPYC 7502-32 CPU, using Microsoft Seal, poly degree of $2^{17}$, and scale $\Delta =2^{40}$.}
\label{tab:golds-sign}
\begin{tabular}{|l|l|c|c|c|c|}
\hline
\textbf{Algorithm} & \textbf{$\alpha$} & 7 & 9 & 11 & 13 \\
\hline
\multirow{4}{*}{\sign}
& Iterations & 7 & 9 & 11 & 12 \\
& Multiplication depth & 14 & 18 & 22 & 24 \\
& Time (s) & 6.3 & 11.7 & 20.2 & 26.9 \\
\hline
\multirow{4}{*}{\invsqr}
& Iterations & 5 & 6 & 7 & 8 \\
& Multiplication depth & 10 & 12 & 14 & 16 \\
& Time (s) & 4.1 & 6.6 & 9.8 & 12.8 \\
\hline
\end{tabular}
\end{table}
\section{Previous, Comparison-Based Methods for argmax over HE}
\label{sec:prev-methods}

\Cref{fig:he-argmax} visualizes the two canonical comparison-based approaches for encrypted $\argmax$: the \emph{tournament} method (\Cref{fig:tour}) and the \emph{league} method (\Cref{fig:league}). Both rely on repeated encrypted comparisons implemented via polynomial approximations of $\sign(\cdot)$.

In the tournament design, approximately $n{-}1$ pairwise comparisons are arranged into $\log_2 n$ sequential stages. At each stage, the number of candidates is halved until only the maximum remains. While asymptotically efficient, this design requires $\log n$ sequential $\sign$ evaluations, each incurring costly rotations and bootstraps.

In the league design, every pair of inputs is compared, and results are accumulated so that the true maximum is the only element with score $n{-}1$. This requires $n$ sequential rounds of comparisons. Although SIMD parallelism allows all $n$ pairwise comparisons within a round to be executed in parallel slots, the \emph{sequential} round structure still dominates, leaving the method quadratic overall.

As summarized in \Cref{tab:golds-sign}, polynomial $\sign$ approximations are substantially deeper and slower than inverse-square-root or division. Thus, even when accelerated with SIMD, both the tournament and league methods remain inefficient at large $n$. In practice, this has limited published HE $\argmax$ implementations either to small input sizes or to systems that decrypt logits before $\argmax$.
\begin{figure}[t!]
    \centering
    \begin{subfigure}[b]{0.48\textwidth}
        \centering
        \includegraphics[width=0.95\linewidth]{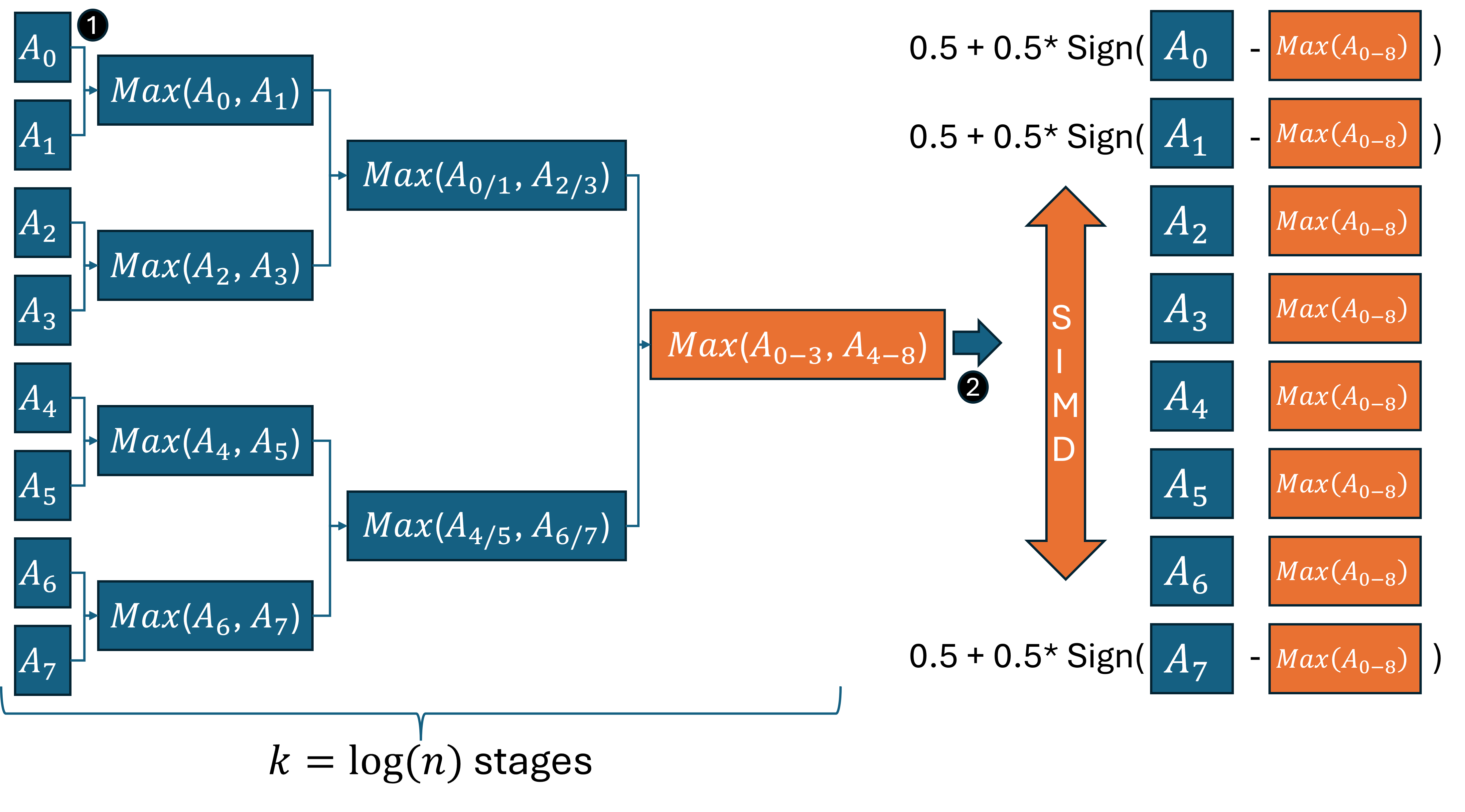}
        \caption{Tournament method.}
        \label{fig:tour}
    \end{subfigure}\hfill
    \begin{subfigure}[b]{0.48\textwidth}
        \centering
        \includegraphics[width=0.95\linewidth]{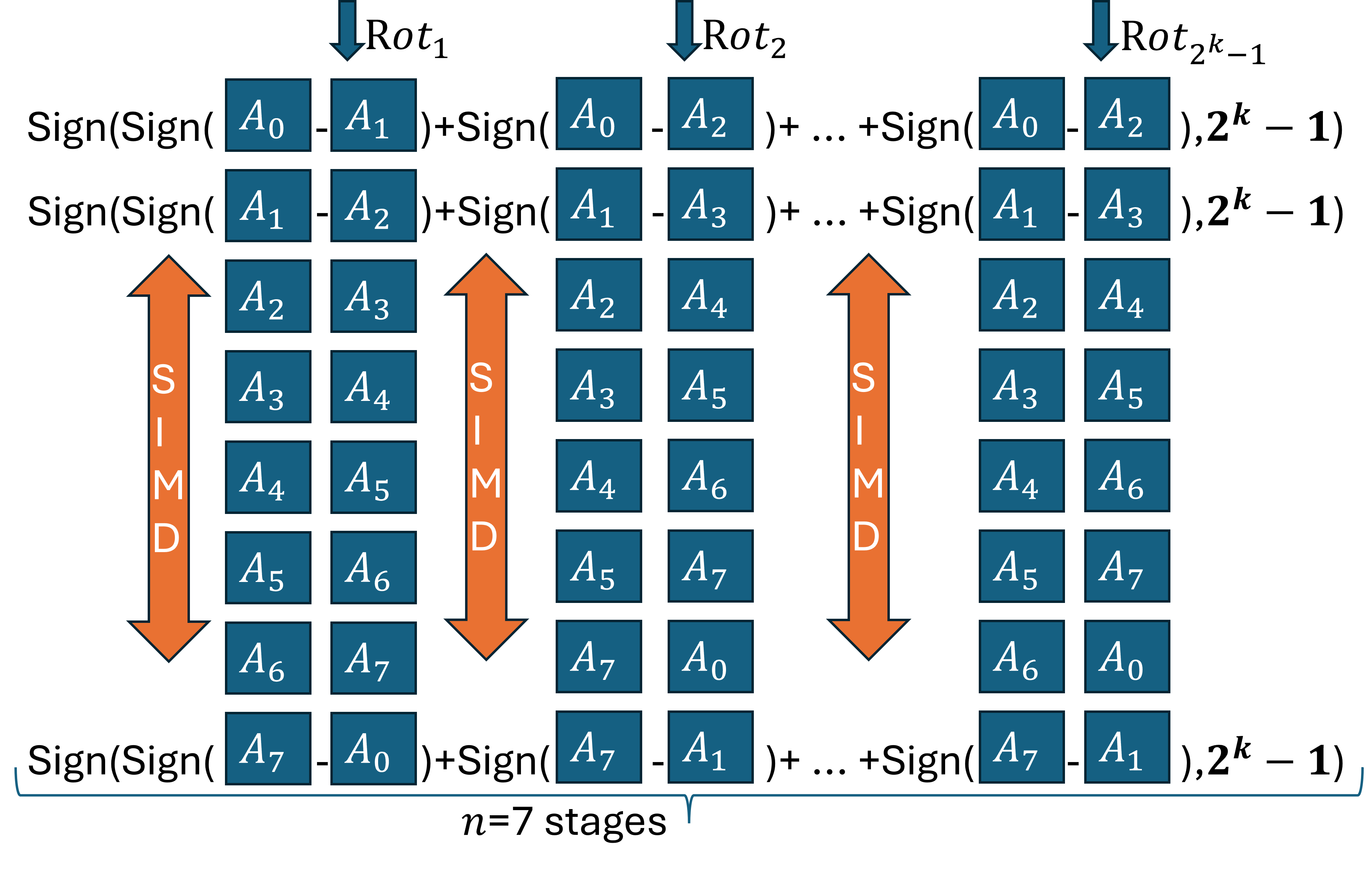}
        \caption{League method.}
        \label{fig:league}
    \end{subfigure}
    \caption{\textbf{Illustration of Prior Comparison-based argmax Methods in HE} (see main text \Cref{sec:argmax}).}
    \label{fig:he-argmax}
\end{figure}
\section{Argmax using SIMD}\label{app:argmaxsimd}
\paragraph{SIMD and Rotations.}
CKKS enables SIMD-style packing, where one ciphertext encrypts a vector $(v_1, \ldots, v_s)$. Operations are applied elementwise, and encrypted vectors can be rotated by $\ell < s$ positions. This allows for patterns such as \rotateandsum~\cite{he4ds} - which aggregates values with $\log s$ rotations and additions.
\begin{figure*}[th!]
    \begin{subfigure}[b]{0.49\textwidth}
        \centering
        \includegraphics[width=0.69\linewidth]{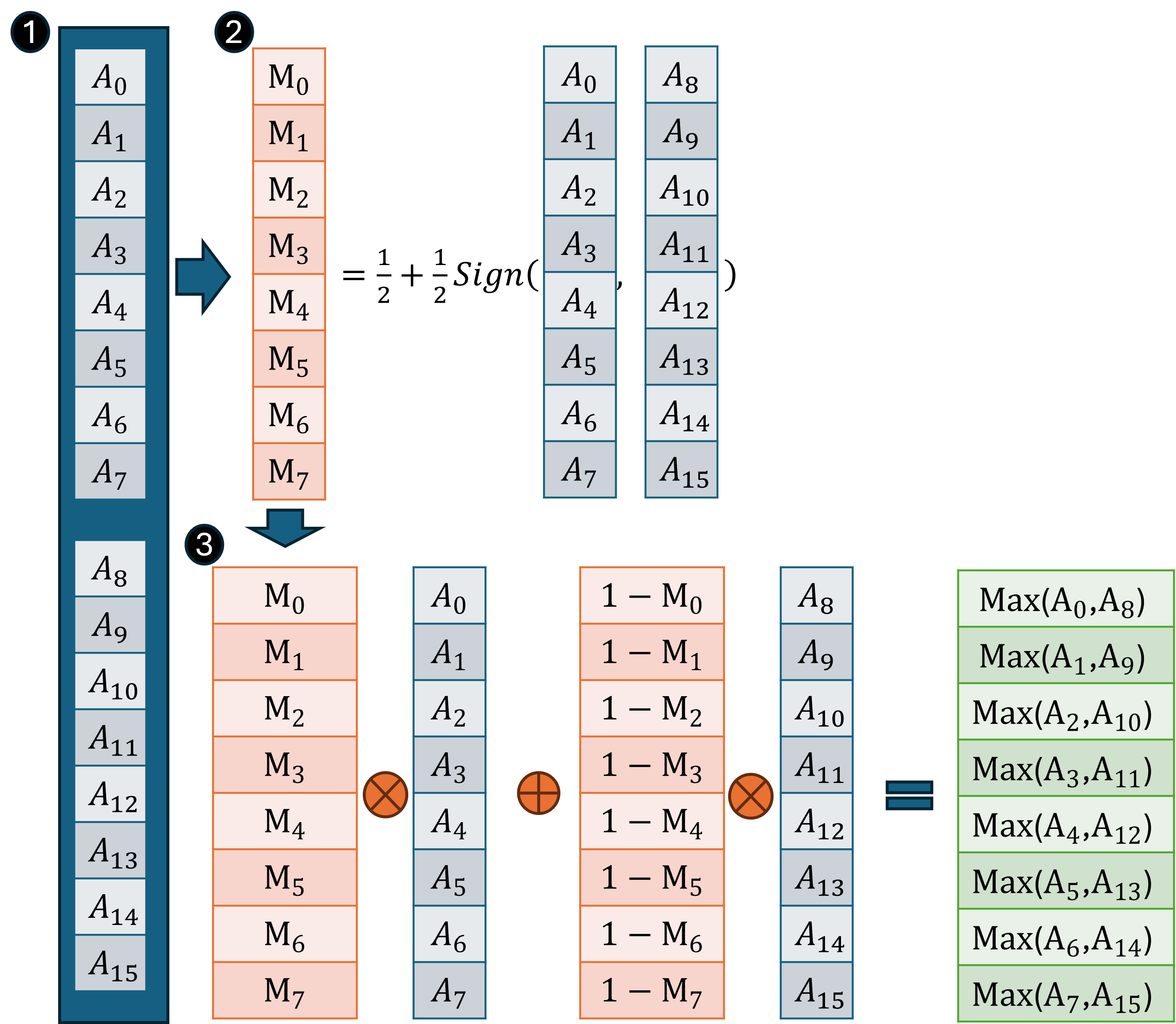}
        \caption{$n > s$.}
        \label{fig:full}
    \end{subfigure}
        \begin{subfigure}[b]{0.49\textwidth}
        \centering
        \includegraphics[width=0.69\linewidth]{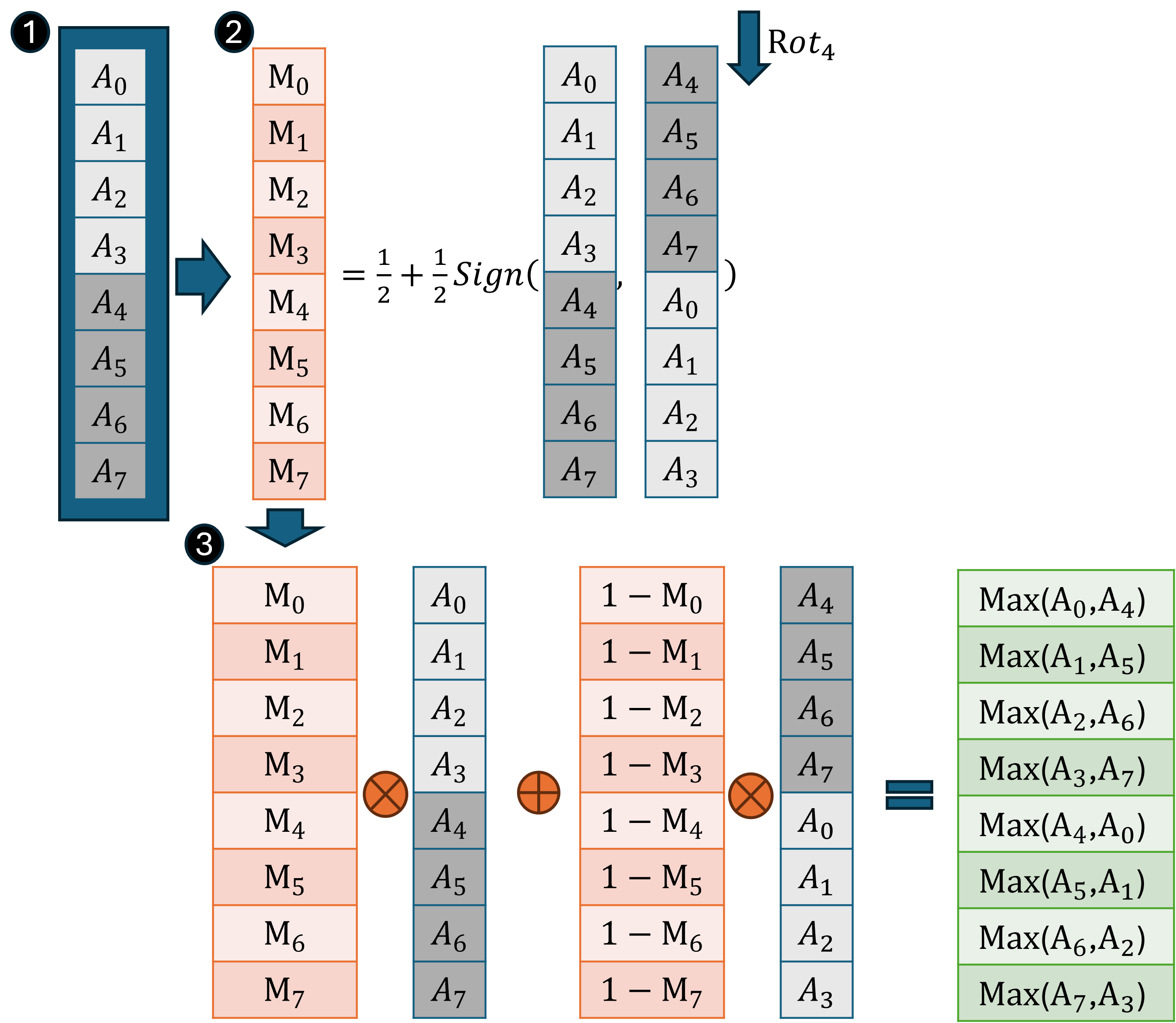}
        \caption{$n=s$.}
        \label{fig:partial}
    \end{subfigure}
    \caption{\gls{HE} \gls{SIMD} example implementation.}
    \label{fig:he-argmax-simd}
\end{figure*}
\Cref{fig:he-argmax-simd} demonstrates how to leverage \gls{SIMD} to execute $n$ $\max(a, b)$ operations in parallel. Step (1) shows the encrypted input; in Step (2), we compute an indicator vector where $M[i] = 1$ if $A[i] > A[i + \frac{n}{2}]$, and 0 otherwise. Subsequently, Step (3) selects (masks) only the elements that match the indicator vector, which returns the desired results.
Consider two cases: $n > s$ (\Cref{fig:full}) and $n = s$ (\Cref{fig:partial}) (The case $n < s$ behaves similarly to $n = s$). Here, $s$ is the maximal number of elements that a single ciphertext can encrypt. \Cref{fig:full} illustrates the case where $n = 16$ elements are encrypted in two ciphertexts, each containing a vector of 8 elements. In contrast, \Cref{fig:partial} shows the case $n = s$, which requires only one ciphertext. In the first case, no rotation is needed, as we can directly subtract one ciphertext from the other using element-wise subtraction. However, this is not the case in \Cref{fig:partial}, where we must homomorphically rotate $A$ by $n/2$ positions. As a result, the output is duplicated, and we only need to keep the top $n/2$ elements.
It appears that many comparisons can be saved by using \gls{SIMD}. Consider the case where $n = s$ for the league method (\Cref{fig:league}) that requires only $n$ homomorphic rotations and $n$ parallel \sign evaluations, with one additional final parallel \sign evaluation. The complexity is thus reduced from $n^2$ to $n$ \sign operations, which, on hardware with sufficient parallel threads, may result in faster execution than the sequential tournament algorithm. When $n>s$ it is worth combining the methods as was done in \cite{BGVCompArgmaxCombined}. Starting with a tournament and when the number of elements is reduced beyond some threshold continue with the league method.

\newpage
\section{Convergence of \cutmax}\label{app:cutmax-convergence-standalone}

\subsection{Preliminaries and Definitions}
The following definitions assume that a vector $x \in \mathbb{R}^n$ has a unique maximum element. We also denote the average of $x$ by $\mu_x = \frac{1}{n}\sum_i x_i$ and its standard deviation by $\sigma_{x} = \sqrt{\frac{1}{n}\sum_i (x_i - \mu_x)^2}$.

\begin{definition}[Max elements]\label{def:maxelems}
Given a vector \( x \in \mathbb{R}^n \) with a unique maximal element, the index of that element is denoted by
$ {i^{*}_{x}} = \argmax_i x_i$ and the index of the second-largest element (excluding the maximum) is denoted by $i^{**}_x = \argmax_{j \neq i^{*}_{x}} x_j$.
\end{definition}

\begin{definition}[$\delta$-gap bound]\label{def:gapcond}
A vector $x \in \mathbb{R}^n$ with a unique maximum at index $i^*_x$ is $\delta$-gap bounded for $\delta \in \mathbb{R}_+$ when
$x_{i^*_x} - x_{i^{**}_x} \geq \delta \cdot \sigma_x$
\end{definition}

\begin{definition}[Gap Ratio]
For $y \in \mathbb{R}^n$, its \emph{gap ratio} is
$\Gamma(y) = \frac{y_{i^*_y}}{y_{i^{**}_y}}$
\end{definition}

\begin{definition}[$c$-semi-standardization]
For $c \in \mathbb{R}_+$, a $c$-semi-standardization function is 
\begin{align}
S_c(y): \mathbb{R}^n \longrightarrow \mathbb{R}^n \\
y_i  \longmapsto  \frac{y_i - \mu_y}{c \cdot \sigma_y} + 1
\end{align}
 \end{definition}

\begin{definition}[\cutmax Iteration]
Given $c \in \mathbb{R}_+$ and an odd $p \in \mathbb{Z}_+$, the \cutmax iteration function is
\begin{align}
\cutmax_t(y^{(t)}): \mathbb{R}^n \longrightarrow \mathbb{R}^n \\ 
y^{(t)}_i  \longmapsto  {\left(S_c (y^{(t)} )\right)}^p
\end{align}
\end{definition}
for some iteration index $t>0$.
 
\subsection{Key Lemmas}

The following lemma shows that running a standardization on a vector $y$ preserves the ordering of elements in it. In addition, it transforms the gap condition.
\begin{lemma}[]
\label{lem:cutmax_monotonic}
The $c$-semi-standardization function ($S_c$) is a strictly increasing function.
\end{lemma}

\begin{proof}

For every $y \in \mathbb{R}^n$ and for every $i, j  \in [n] \,, i \neq j$, because $\mu_y, \sigma_y, c >0$ it follows that
\begin{align}
y_i > y_j &\iff y_i - \mu_y > y_j - \mu_y \\
          &\iff \frac{y_i - \mu_y}{c \cdot \sigma_y} > \frac{y_j - \mu_y}{c \cdot \sigma_y} 
          \iff S_c(y)_i > S_c(y)_j
\end{align}
\end{proof}

\begin{lemma}[]
\label{lem:preservegap}
Let $y \in \mathbb{R}^n$ be a $\delta$-gap bounded vector then $S_c(y)$ is a $\frac{\delta}{c}$-gap bounded vector.
\end{lemma}

\begin{proof}
Set $z := S_c(y)$, $i^*_z = i^*_y$ and $i^{**}_z = i^{**}_y$. then
\begin{equation}
z_{i^*_z} - z_{i^{**}_z} = \frac{y_{i^*_y} - y_{i^{**}_y}}{c \cdot \sigma_y} \underset{\text{$y$ is $\delta$-gap bounded}}{\ge} \frac{\delta \cdot \sigma_y}{c \cdot \sigma_y} = \frac{\delta}{c}
\end{equation}
\end{proof}

The next Lemma shows that applying a $c$-semi-standardization function $S_c$ on some vector $y$ results with negative coefficients in places where $y_i$ is at least $c$ standard deviations below the mean of $y$.

\begin{lemma}[Unconditional gap growth for one CutMax step]\label{lem:cutmax-unconditional}
Let $\yT\in\mathbb{R}^n$ have a unique maximizer at $i^*$ and let $i^{**}\neq i^*$ be the runner-up index.
Assume $\yT$ is $\delta$-gap bounded in the sense that
\(
\yT_{i^*}-\yT_{i^{**}}\ge \delta\,\sigma_{\yT}
\)
(Def.~\ref{def:gapcond}), where $\sigma_{\yT}$ denotes the population standard deviation.
Assume additionally that $\min_i y^{(t)}_i \geq \mu_{y^{(t)}} - c \sigma_{y^{(t)}} + \epsilon$ for small $\epsilon > 0$, ensuring all $z_i > 0$.
Fix $c>0$ and an odd integer $p\ge 1$, and set
$
z=S_c(\yT)=\frac{\yT-\mu_{\yT}}{c\,\sigma_{\yT}}+1
$
and
$
\yTp = z^{\odot p}
$
(the elementwise odd-power step of CutMax).
Then the post-step gap ratio satisfies
$
\frac{\yTp_{i^*}}{\yTp_{i^{**}}}
= \left(\frac{z_{i^*}}{z_{i^{**}}}\right)^{\!p}
\;\;\ge\;\;
\left(1+\frac{\delta}{\,c+\sqrt{n}\,}\right)^{\!p}.
$
Consequently, after $T$ CutMax iterations,
$
\frac{y^{(T)}_{i^*}}{y^{(T)}_{i^{**}}}
\;\ge\;
\left(1+\frac{\delta}{\,c+\sqrt{n}\,}\right)^{\!pT}.
$
\end{lemma}

\begin{proof}
By Lemma~\ref{lem:preservegap},
\(
z_{i^*}-z_{i^{**}}\ge \delta/c.
\)
Therefore
\[
\frac{z_{i^*}}{z_{i^{**}}}
=1+\frac{z_{i^*}-z_{i^{**}}}{z_{i^{**}}}
\;\ge\;
1+\frac{\delta}{c\,z_{i^{**}}}.
\]
To upper-bound $z_{i^{**}}$, write $r_i=\frac{\yT_i-\mu_{\yT}}{\sigma_{\yT}}$.
From the definition of the population standard deviation, we have
\(
\frac1n\sum_i r_i^2=1
\Rightarrow
\sum_i r_i^2=n
\),
hence
\(
r_{i^{**}}\le \sqrt{n}
\)
and thus
\(
z_{i^{**}}=1+\frac{r_{i^{**}}}{c}\le 1+\frac{\sqrt{n}}{c}.
\)
Combining the bounds gives
\[
\frac{z_{i^*}}{z_{i^{**}}}
\;\ge\;
1+\frac{\delta}{c\left(1+\sqrt{n}/c\right)}
=
1+\frac{\delta}{\,c+\sqrt{n}\,}.
\]
Finally, the CutMax power step yields
\(
\frac{\yTp_{i^*}}{\yTp_{i^{**}}}
=
\left(\frac{z_{i^*}}{z_{i^{**}}}\right)^{p}
\),
which gives the stated inequality. Iterating $T$ times multiplies the exponent by $T$.
\end{proof}

\begin{lemma}[Per-step gap amplification]\label{lem:gap-growth}
Let $\mathbf y^{(t)}\in\mathbb R^{n}$ satisfy the $\delta$–gap condition
{\normalfont(Def.~\ref{def:gapcond})}, fix $c>0$ and odd $p\ge1$, and set
$z=S_c(\mathbf y^{(t)})$ and $\mathbf y^{(t+1)}=z^{\odot p}$.
Then the gap ratio satisfies
\[
  \Gamma^{(t+1)}
  \;=\;
  \Gamma^{(t)}\!\left(\frac{z_{i^*}}{z_{i^{**}}}\right)^{\!p}
  \;\ge\;
  \Gamma^{(t)}\!\left(1+\frac{\delta}{\,c+\sqrt{n}\,}\right)^{\!p}.
\]
\end{lemma}

\begin{proof}
By Lemma~\ref{lem:cutmax-unconditional},
\(
  \frac{z_{i^*}}{z_{i^{**}}}\ge 1+\frac{\delta}{\,c+\sqrt{n}\,}
\).
Since $\Gamma^{(t+1)}=\Gamma^{(t)}(z_{i^*}/z_{i^{**}})^{p}$, the claim follows.
\end{proof}

\begin{theorem}[CutMax Convergence]\label{thm:cutmax-convergence}
Let $\mathbf{x}\in\mathbb R^{n}$ satisfy the $\delta$–gap condition with $\delta>0$,
and run CutMax with parameters $(p,c,T)$, $c>0$, odd $p\ge1$.
Then
\[
  \Gamma^{(T)}
  \;\ge\;
  \Gamma^{(0)}\!\left(1+\frac{\delta}{\,c+\sqrt{n}\,}\right)^{\!pT}.
\]
In particular, to ensure $\Gamma^{(T)}\ge k$, it suffices to take
\[
  T \;\ge\;
  \left\lceil
    \frac{\ln k - \ln \Gamma^{(0)}}{\,p\,\ln\!\bigl(1+\frac{\delta}{\,c+\sqrt{n}\,}\bigr)}
  \right\rceil
  \;\;\le\;\;
  \left\lceil
    \frac{\ln k}{\,p\,\ln\!\bigl(1+\frac{\delta}{\,c+\sqrt{n}\,}\bigr)}
  \right\rceil,
\]
where the last inequality uses $\Gamma^{(0)}\ge1$.
\end{theorem}

\begin{proof}
Iterating Lemma~\ref{lem:gap-growth} yields
\(
\Gamma^{(T)}
\ge \Gamma^{(0)}\!\left(1+\frac{\delta}{\,c+\sqrt{n}\,}\right)^{pT}.
\)
Solving for $T$ gives the stated bound.
\end{proof}

\begin{table}[h]
\centering
\small
\caption{Min.\ iterations $T_{\min}$ for $n=32{,}768$ (Mistral-7B), $p=13$, $c=5$.}
\begin{tabular}{@{}lcc@{}}
\toprule
$\delta$ & $k=10$ & $k=100$ \\
\midrule
0.5 &  66 & 132 \\
1   &  34 &  67 \\
2   &  17 &  34 \\
3   &  12 &  23 \\
5   &   7 &  14 \\
10  &   4 &   7 \\
\bottomrule
\end{tabular}
\end{table}
\end{document}